%% file: main.tex
\documentclass[onecolumn,11pt, journal]{IEEEtran}
\renewcommand*\footnoterule{}
\usepackage[round]{natbib}

\input{PackageLoad2021}

\input{Notation_header}
\input{acronyms.tex}

\settoggle{Track}{false}
\settoggle{Blue_revision}{false}

\newcommand{\addvb}[1]{\addv{b}{off}{#1}}

\makeatletter
\def\footnoterule{\kern-3\p@
  \hrule \@width 6.5in \kern 2.6\p@} 
\makeatother
\begin{document}
\setlength{\baselineskip}{16pt}	
%

%

\title{Agnostic PAC Learning of $k$-juntas Using $L_2$-Polynomial Regression}

\author{

\IEEEauthorblockN{ Mohsen Heidari \IEEEauthorrefmark{1} and  Wojciech Szpankowski \IEEEauthorrefmark{2}}\\
       \IEEEauthorrefmark{1}   Department of Computer Science, Indiana University, Bloomington \\
      \IEEEauthorrefmark{2} 
Department of Computer Science, Purdue University\\
\IEEEauthorrefmark{1}\tt mheidar@iu.edu,
 \IEEEauthorrefmark{2}\tt szpan@purdue.edu
   \thanks{This work was partially supported by the NSF Center for Science of Information (CSoI) Grant
CCF-0939370, and also by NSF Grants CCF-2006440, CCF-2007238, CCF-2211423, and Google Research Award.}   }

\IEEEoverridecommandlockouts
\maketitle

%
%
%
%
%

\begin{abstract}%
  Many conventional learning algorithms rely on loss functions other than the natural 0-1 loss for computational efficiency and theoretical tractability.  Among them are approaches based on  absolute loss ($\mathcal{L}_1$ regression) and  square loss ($\mathcal{L}_2$ regression). The first is proved to be an \textit{agnostic} PAC learner for various important concept classes such as \textit{juntas}, and  \textit{half-spaces}. On the other hand, the second is preferable because of its computational efficiency, which is linear in the sample size. However, PAC learnability is still unknown as guarantees have been proved only under distributional restrictions. The question of whether  $\mathcal{L}_2$ regression is an agnostic PAC learner for 0-1 loss has been open since 1993 and yet has to be answered.  
  
  This paper resolves this problem for the junta class on the Boolean cube --- proving agnostic PAC learning of $k$-juntas using  $\mathcal{L}_2$ polynomial regression. Moreover, we present a new PAC learning algorithm based on the Boolean Fourier expansion with lower computational complexity. 
  Fourier-based algorithms, such as \citet{linial1993constant}, have been used under distributional restrictions, such as uniform distribution. We show that with an appropriate change, one can apply those algorithms in agnostic settings without any distributional assumption.    We prove our results by connecting the PAC learning with 0-1 loss to the \ac{MMSE} problem. We derive an elegant upper bound on the 0-1 loss in terms of the MMSE error.  Based on that, we show that the sign of the \ac{MMSE} is a PAC learner for any concept class containing it.
  \end{abstract}

\input{Intro}

\section{Formulations and Main Results}\label{sec:main results}
\noindent\textbf{Model:} We use the usual formulation of \textit{agnostic} \ac{PAC} learning model \citep{Valiant1984,Kearns1994}. This paper focuses on binary classification with the 0-1 loss. 
 An algorithm  agnostically PAC learns a hypothesis class $\mathcal{H}$, if, for any $\epsilon, \delta \in (0,1)$, and given $n>n(\epsilon,\delta)$  training samples drawn from any distribution $D$, it outputs with probability $(1-\delta)$ a predictor $g$ whose expected loss is at most $\Lopt+\epsilon$,
 where  $\Lopt$ is the minimum  loss in $\mathcal{H}$. 

\noindent{\bf Notation:} For any natural number $d$, the set $\{1,2,\cdots, d\}$ is denoted by $[d]$.  For a pair of functions $f,g$ on $\mathcal{X}$, the notation $f\equiv g$ means that $f(x)=g(x)$ for all $x\in\mathcal{X}$. For any function $h: \mathcal{Z}\rightarrow \RR$ and input distribution $D$, the $1$-norm and $2$-norm  are defined as $\norm{h}_{1, D}:= \EE_D[|h(Z)|]$ and $\norm{h}_{2, D}:=\sqrt{\EE_D[h(Z)^2]}$, respectively.

%

\begin{algorithm}[t]
\caption{ $\Ltwo$-Algorithm}
\label{alg:L2}
\DontPrintSemicolon
\KwIn{Training samples $\mathcal{S}_n=\{(\bfx(i), y(i))\}_{i=1}^n$, degree parameter $k$.}
 \SetKwProg{Fn}{}{:}{} 
\SetKwFunction{main}{LearningAlgorithm}

\For{each subset  $\mathcal{J}\subseteq [d]$ with $|\mathcal{J}|=k$}{
 Find a polynomial $\hat{p}_{\mathcal{J}}$ of degree up to $k$ that minimizes  $\frac{1}{n}\sum_i \big(y(i)-p(\bfx(i))\big)^2$.\;
}
  Select $\hat{p}$ as the $\hat{p}_{\mathcal{J}}$ that has the smallest square loss.\;

\KwRet   $\hat{g}\equiv \sign[\hat{p}]$.\;
\end{algorithm}

\subsection{Warm-Up}\label{sec:sign}
We start with highlighting one of the main difficulties in proving PAC bounds with $\mathcal{L}_1$ or $\mathcal{L}_2$ regression. The main challenge is analyzing the 0-1 loss after taking the $\sign$ of the resulting polynomial. For that, one needs to study the relations between the 0-1 loss and the square or absolute loss.  To see this, let $p$ be the polynomial minimizing the square loss. Then, it is not difficult to see that $\11\set{y\neq \sign[p(x)]}\leq (y-p(x))^2$, where $y\in \pmm$. As a result, the 0-1 loss of $\sign[p]$ is bounded as $\PP\big\{Y\neq \sign[p(\bfX)]\big\}\leq \EE[(Y-p(\bfX))^2].$ This is a loose bound that leads to a PAC bound of $8\Popt$. To see the argument, let $f$ be the optimal predictor with the 0-1 loss $\Popt$. Additionally, suppose $f$ is approximated by a polynomial $\tilde{p}$ with the square error less than $\epsilon^2$. Then, we can write
\begin{align*}
\PP\big\{Y &\neq \sign[p(\bfX)]\big\}\leq \EE[(Y-p(\bfX))^2]\\
&\stackrel{(a)}{\leq} \EE[(Y-\tilde{p}(\bfX))^2]\\
& \stackrel{(b)}{\leq} 2\EE[(Y-f(\bfX))^2 + (f(\bfX)-\tilde{p}(\bfX))^2]\\
& \stackrel{(c)}{\leq}  8\Popt+2\epsilon^2,
\end{align*}
where (a) follows as $p$ is the optimal polynomial, (b) holds from the AM-GM inequality, and (c) holds as $\11\set{y\neq f(x)}=\frac{1}{4}(y-f)^2$ for any $f:\mathcal{X}\rightarrow \pmm$.

These observations raise whether taking the $\sign$ is optimal in PAC learning. When $x\in \pmm^d$, Blum \etal~ \citep{Blum1994} and Jackson \citep{Jackson2006} proposed a clever idea of randomized rounding instead of taking the sign. As a result, they improved the factor from $8\Popt$ to $2\Popt$. 

In  Lemma \ref{lem:norm2 sign}, we prove a tighter bound between the 0-1 loss and the square loss. Using this lemma, we demonstrate in Theorem \ref{thm:L2 PAC} that for $k$-junta class taking the sign is optimal and results in $\Popt$ (i.e., agnostic PAC learnability).  Moreover, we develop a more general analysis and show that sign of the MMSE of $Y$ given the observation $\bfX$ give a PAC learner,  see Theorem \ref{thm:MMSE PAC}.



\subsection{Learning  with $\Ltwo$-Polynomial Regression}\label{sub:L2poly PAC}
We employ a PAC learning algorithm using $\Ltwo$-polynomial regression. Given a training set, the objective of the polynomial regression is to minimize the empirical square loss over all polynomials of degrees up to $k$. This process can be implemented by stochastic gradient descent or solving a linear equations system. 
Based on this regression, one can study PAC learning of various concept classes. In this paper, we consider $k$-juntas. 

\noindent{\bf $k$-junta class:} A $k$-junta is a Boolean function $h:\pmm^d\rightarrow \pmm$ with $d$ input variables whose output depends on at most $k$ out of $d$ inputs. 

For $k$-junta classes, we use a variant of $\Ltwo$-polynomial regression (see Algorithm \ref{alg:L2})  that has the same computational complexity as compared to the vanilla $\mathcal{L}_2$ polynomial regression. 
With this approach, we establish the following theorem. 
 



\begin{theorem}\label{thm:L2 PAC}
Algorithm \ref{alg:L2}, with a degree limit of $k \leq d$, agnostically PAC learns $k$-juntas.  More precisely, 
given $\delta\in [0,1]$, with probability $(1-\delta)$, its generalization loss does not exceed the following  
\begin{equation*}
\Pek + O\big(\sqrt{\frac{2^k+ k\log d}{n}\log \frac{n}{2^k+ k\log d}}\big)+ \sqrt{\frac{\log(1/\delta)}{2n}},
\end{equation*}
where $n$ is the number of samples. Furthermore, the resulting computational complexity is  $O(n d^{\Theta(k)})$.
\end{theorem}
By simplifying the above expression, we  get a sample complexity bound of $n(\delta, \epsilon) = O(\frac{k2^k}{\epsilon^2} \log\frac{d}{\epsilon^2 \delta})$.
The proof is presented in Section \ref{subsec:proof:thm:L2 PAC}.

The polynomial regression procedure in Algorithm \ref{alg:L2} can be implemented via a linear $L_2$ regression in $\RR^{D}$, where $D=d^k$. The factor $d^k$ is because a polynomial of degree up to $k$ is a linear combination of monomials of the form $\prod_{j=1}^k X_{i_j}$, where $i_j\in [d]$.   Linear regression can be implemented via Moore-Penrose (generalized) inverse. The generalized inverse is computed using classical methods in $O(nD^2+D^3)$. Hence, given that $n>D$, we can perform polynomial regression in $O(n D^2)=O(n d^{2k})$. However, we note that under special cases (e.g., $d^k=\lceil{n^r}\rceil)$ the regression can be done in $O(nd^{k\omega(r)})$, where $\omega(r)$ is a constant given in \citep{Gall2018}. Hence, the computational complexity of this algorithm is  $O(nd^{\Theta(k)})$ as noted in Table \ref{tab:PAC compare}.

\input{FourierOrth.tex}

\input{analysis.tex}


\input{Fourier_Proof.tex}

\subsection{PAC Learning with MMSE}
Lastly, we discuss a more general indication of our result about the PAC learnability of MMSE.

\begin{theorem}\label{thm:MMSE PAC}
Suppose $\mathcal{A}$ is an algorithm that outputs  $\sign[\hat{Y}_{MMSE}]$, where $\hat{Y}_{MMSE}$ is the empirical MMSE of $Y$ given the observation samples $\mathcal{S}_n$. Then, $\mathcal{A}$ agnostically PAC learns any concept class $\mathcal{C}$  containing $\sign[\hat{Y}_{MMSE}]$ with error up to 
\begin{align*}
opt_{\mathcal{C}}+ O\Big(\sqrt{\frac{VC}{n}\log(\frac{n}{\delta VC}}\Big),
\end{align*}
where $VC$ is the VC dimension of $\mathcal{C}$.
\end{theorem}
This result is a consequence of  Lemma \ref{lem:norm2 sign} applied to empirical loss followed by VC theory.

 \input{generalhypothesis.tex}

\section*{Conclusion}
This paper studies PAC learning using algorithms based on $\mathcal{L}_2$ polynomial regression. Mainly, we show that $\mathcal{L}_2$ based algorithms are PAC learners for the $k$-junta class. Moreover, we present a more efficient PAC learning algorithm  based on the (uniform) Boolean Fourier expansion.
Our approach relies on two frameworks, one  connecting MMSE and PAC and the other connecting PAC and the Boolean Fourier expansion.   With this approach and powerful tools for analyzing vector spaces, we derive tighter bounds between the 0-1 loss and the square loss. 



\bibliography{main}

\newpage
\onecolumn
\appendices

\input{Proofs/proof_lemmMMSE}
 \input{Proofs/proof_lem_norm2sign}
\input{Proofs/proof_lem_norm2fourier}

\input{Proofs/proof_thm_L2_polynomials}
\input{Proofs/proof_lem_Pe_theta}

\end{document}

%% file: PackageLoad2021.tex
\usepackage{amsmath}
\usepackage{amsfonts}
\usepackage{amsthm}   
\usepackage{mathrsfs}
\usepackage[utf8]{inputenc}

\usepackage{epstopdf}
\usepackage{mathtools}
\usepackage{graphicx}
\usepackage{bbm}
\usepackage{enumerate}
\usepackage{epsf,verbatim,amssymb,array,cite,multicol,multirow}  
\usepackage{psfrag,bm,xspace}
\usepackage{hhline}
\usepackage{xstring}
\usepackage{ifthen}

\usepackage{booktabs} 

\usepackage[linesnumbered,ruled,vlined,procnumbered]{algorithm2e}

\ifdefined \theorem 
\else
  \newtheorem{theorem}{Theorem}
\fi
\ifdefined \lemma 
\else
\newtheorem{lemma}{Lemma}
\fi
\ifdefined \corollary 
\else
\newtheorem{corollary}{Corollary}
\fi

\newtheorem{fact}{Fact}

\ifdefined \proposition 
\else
\newtheorem{proposition}{Proposition}
\fi

\ifdefined \definition 
\else
\newtheorem{definition}{Definition}
\fi

\ifdefined \example 
\else
\newtheorem{example}{Example}
\fi

\ifdefined \remark 
\else
\newtheorem{remark}{Remark}
\fi

\newtheorem{innercustomgeneric}{\customgenericname}
\providecommand{\customgenericname}{}
\newcommand{\newcustomtheorem}[2]{%
  \newenvironment{#1}[1]
  {%
   \renewcommand\customgenericname{#2}%
   \renewcommand\theinnercustomgeneric{##1}%
   \innercustomgeneric
  }
  {\endinnercustomgeneric}
}

\newcustomtheorem{customthm}{Theorem}
\newcustomtheorem{customlemma}{Lemma}

\makeatletter
\def\old@comma{,}
\catcode`\,=13
\def,{%
  \ifmmode%
    \old@comma\discretionary{}{}{}%
  \else%
    \old@comma%
  \fi%
}
\makeatother

\newcommand\numberthis{\addtocounter{equation}{1}\tag{\theequation}}

\usepackage{xcolor}
\definecolor{darkblue}{rgb}{0.1,0.1,0.8}
\definecolor{DarkGreen}{rgb}{0,0.6,0}
\definecolor{brickred}{rgb}{0.8, 0.25, 0.33}
\definecolor{britishracinggreen}{rgb}{0.0, 0.26, 0.15}
\definecolor{calpolypomonagreen}{rgb}{0.12, 0.3, 0.17}
\definecolor{ao(english)}{rgb}{0.0, 0.5, 0.0}
	\definecolor{cadmiumgreen}{rgb}{0.0, 0.42, 0.24}
\definecolor{burgundy}{rgb}{0.5, 0.0, 0.13}
\usepackage{etoolbox}

\providetoggle{Blue_revision}
\settoggle{Blue_revision}{true}

\providetoggle{Track}
\settoggle{Track}{true}
\newcommand{\addv}[3]{%
	\iftoggle{Track}{%
    	\IfEqCase{#1}{%
       	 	{a}{\ifthenelse{\equal{#2}{ON}}{{\color{cadmiumgreen}#3}}{#3}}%
        	{b}{\ifthenelse{\equal{#2}{ON}}{{\color{brickred}#3}}{#3}}%
       		{c}{\ifthenelse{\equal{#2}{ON}}{{\color{burgundy}#3}}{#3}}%
       		{d}{\ifthenelse{\equal{#2}{ON}}{{\color{red}#3}}{#3}}
    	}[\PackageError{tree}{Undefined option to tree: #1}{}]%
	}{#3}%
}

 \usepackage{hyperref}

\hypersetup{
colorlinks=true, %
  pdfstartview={FitH},
    linkcolor=red,
    citecolor=blue,
    urlcolor={blue!80!black}
}

\usepackage{graphicx}

\newcounter{relctr} 
\everydisplay\expandafter{\the\everydisplay\setcounter{relctr}{0}} 

\AtBeginDocument{} 

%% file: Notation_header.tex
\def\etal{\textit{et al.}}

\global\long\def\RR{\mathbb{R}}

\global\long\def\NN{\mathbb{N}}

\global\long\def\EE{\mathbb{E}}
\global\long\def\PP{\mathbb{P}}

\usepackage{dsfont}
\global\long\def\11{\mathds{1}}

\newcommand{\bfx}{\mathbf{x}}

\newcommand{\bfX}{\mathbf{X}}

\global\long\def\+{\oplus}

\newcommand\pmm{\{-1,1\}}

\newcommand{\prob}[1]{\PP\Big\{  #1 \Big\} }
\def\<{\langle}
\def\>{\rangle}

\DeclareMathOperator*{\sign}{sign}

\ifdefined \var 
  \renewcommand{\var}{\mathsf{var}}
\else
  \newcommand{\var}{\mathsf{var}}
\fi

\ifdefined \abs \else
 \newcommand{\abs}[1]{\lvert#1\rvert}
\fi

\ifdefined \norm \else
 \newcommand{\norm}[1]{\lVert#1\rVert}
\fi

\ifdefined \set 
  \renewcommand{\set}[1]{\left\{#1\right\}}
\else
  \newcommand{\set}[1]{\left\{#1\right\}}
\fi

\newcommand*{\medcup}{\mathbin{\scalebox{1}{\ensuremath{\bigcup}}}}%

\DeclareMathOperator*{\argmin}{arg\,min}

\usepackage{stackengine}
\def\deq{\mathrel{\ensurestackMath{\stackon[1pt]{=}{\scriptstyle\Delta}}}}

\def\fS{{f}_{\mathcal{S}}}
\def\fJ{f^{\mathcal{J}}}

\def\hJ{h_{\mathcal{J}}}

\def\gS{g_{\mathcal{S}}}

\def\ps{\chi_\mathcal{S}}
\def\pS{\ps}

\def\festJ{\hat{f}^{\mathcal{J}}}

\def\Jalg{\hat{\mathcal{J}}}
\def\fJalg{\hat{f}^{\Jalg}}

\def\predict{\hat{g}}

\def\LD{\mathcal{L}_2(D)}

\def\L2{\LD}

\def\Loss{\mathcal{L}_D}
\def\loss{\Loss}
\newcommand{\Lopt}{\emph{opt}}

\def\Lossemp{\mathcal{L}_{\hat{D}}}

\newcommand{\Popt}{\Lopt}

\newcommand{\Pek}{\Lopt}

\def\Ltwo{\mathcal{L}_2}

\def\Pemp{{\PP_{\Demp}}}
\def\Demp{\hat{D}}

\def\Pekemp{\widehat{opt}}
\def\Pk{\mathcal{P}_k}

\def\PiJ{\Pi_Y^{\mathcal{J}}}

\def\aS{a_{\mathcal{S}}}
\def\ahatS{\hat{a}_{\mathcal{S}}}
\def\unif{\emph{unif}}

%% file: acronyms.tex
\usepackage[nolist,nohyperlinks]{acronym}

\newacro{ptp}[PtP]{Point-to-Point}
\newacro{iid}[i.i.d.]{independent and identically distributed} 
\newacro{IID}[i.i.d.]{independent and identically distributed} 
\newacro{UFFS}[UFFS]{Unsupervised Fourier Feature Selection}
\newacro{SFFS}[SFFS]{Supervised Fourier Feature Selection}
\newacro{LS}[LS]{Laplacian Score}
\newacro{MAE}[MAE]{mean absolute error}
\newacro{MSE}[MSE]{mean square error}
\newacro{MMSE}[MMSE]{minimum mean square estimation}
\newacro{PAC}[PAC]{\textit{probably approximately correct}}
\newacro{VC}[VC]{Vapnik–Chervonenkis}
\newacro{ERM}[ERM]{Empirical Risk Minimization}
\newacro{SVM}[SVM]{support-vector machine}
\newacro{POVM}[POVM]{positive operator-valued measure}
\newacro{QLD}[QLD]{quantum low-degree algorithm}
\newacro{QNN}[QNN]{quantum neural networks}
\newacro{SGD}[SGD]{stochastic gradient descent}

%% file: Intro.tex
\section{Introduction}
To gain computational efficiency or analytic tractability, many conventional learning methods such as \ac{SVM} rely on intermediate loss functions other than the natural 0-1 loss. Absolute difference ($\mathcal{L}_1$ distance) is an example.  It is known that polynomial regression under $\mathcal{L}_1$ distance leads to \textit{agnostic} \ac{PAC} learners \citep{Kalai2008} for various hypothesis classes such as $k$-juntas, \textit{polynomial-approximated} predictors, and \textit{half-spaces}. However,  the running time of computing $\mathcal{L}_1$ distance is quadratic in sample size and hence prohibitive for large data sets. 

Square loss  ($\mathcal{L}_2$ distance), on the other hand, is an alternative with computational complexity linear in the size of the data. This has been an incentive to use learning algorithms  such as the \textit{low-degree} algorithm \citep{linial1993constant} and LS-SVM \citep{Suykens1999}. From the learning theoretic perspective,  PAC learning using $\mathcal{L}_2$-based approaches has been studied for the aforementioned concept classes, but with distributional assumptions  \citep{linial1993constant,Kalai2008,Jackson2006}. 

For instance, under the \textit{realizability} assumption, where zero generalization loss is possible ($opt=0$), the $\Ltwo$-polynomial regression is a PAC learner. In addition to the realizability assumption, under the uniform input distribution, the low-degree algorithm is also a PAC learner \citep{mossel2004learning,Mossel_ODonnell,blais2010polynomial}. Under the distribution-free (\textit{agnostic)} setting, PAC bounds of the form $c(\Popt)$ with $\Popt$ being the minimum loss of the class and $c$ a constant as high as $c=8$, have been proved so far for various concept classes \citep{Kalai2008,Kearns1994,Jackson2006}.  
Therefore, \textit{agnostic} PAC learnability of $\mathcal{L}_2$-based approaches is still open and yet to be determined. 

This paper resolves this problem for learning $k$-juntas on the Boolean cube, i.e., Boolean functions over $d$ inputs whose output depends on at most $k<d$ variables, where $k$ is typically a constant much smaller than $d$. Learning juntas has been studied extensively in the literature with  various motivations such as feature selection in machine learning \citep{Guyon2003,blais2010polynomial,HeidariICML2021,Kalai2008,Klivans2009,Birnbaum2012,Diakonikolas2019}.
We prove that agnostic PAC learning is possible using $\Ltwo$-polynomial regression for $k$-juntas. Moreover, we present a more efficient variant of $\Ltwo$ regression using a  Boolean Fourier expansion. We show that this algorithm is also an agnostic PAC learner with respect to $k$-juntas.  This result implies that Fourier algorithms such as the low-degree algorithm of \citet{linial1993constant} that were initially designed for uniform distribution also apply to agnostic settings. 

\subsection{Summary of the Contributions}
\noindent\textbf{Learning $k$-juntas with least square regression:} The focus of this paper is PAC learning of $k$-junta class, on Boolean inputs, using $\mathcal{L}_2$-regression and with the usual 0-1 loss. Following the standard \ac{PAC} learning model, the training set contains $n$ samples $\set{(\bfx(i),y(i))}_{i=1}^n$ with feature-vectors $\bfx(i)\in \pmm^d$ and binary labels $y(i)\in \pmm$.   The objective of the $\Ltwo$-polynomial regression is to minimize the empirical square loss between the target label $y$ and a polynomial $p(\bfx)$ of degree up to $k$. Given such a polynomial, a predictor $g$ is created by simply taking the sign of this polynomial as $g(\bfx) = \sign[p(x)]$.

The first main result of this paper shows that $\Ltwo$ polynomial regression agnostically PAC learns $k$-juntas. More precisely, with probability at least $(1-\delta)$, the generalization loss of the predictor $g$ is within a small deviation of the optimal loss among all $k$-juntas, i.e., 
$\PP\{Y\neq g(\bfX)\} \leq \Popt +\epsilon,$ with $\Popt$ being the optimal loss in $k$-junta class. 
More formally, we prove the following theorem.

\begin{customthm}{1 (abbreviated)}
Given $k \leq d$, there is an algorithm based on $\Ltwo$-polynomial regression with degree limit  $k$ (Algorithm \ref{alg:L2}) that agnostically PAC learns $k$-juntas with sample complexity $n$ up to $O(\frac{k2^k}{\epsilon^2}\log \frac{d}{\delta \epsilon^2})$ and computational complexity  $O(n d^{\Theta(k)})$.
\end{customthm}

 We note the computational complexity of learning $k$-juntas with the $\mathcal{L}_1$-polynomial regression is $O(n^2d^{(3+\omega)3k})$ which is worse for large $n$.



One of the main technical challenges in proving PAC bounds with $\mathcal{L}_1$ or $\mathcal{L}_2$ regression is analyzing the connections between the 0-1 loss and the square or absolute loss. Conventional results for $\mathcal{L}_2$ rely on the inequality $\11\set{y\neq \sign[p(x)]}\leq (y-p(x))^2$ that holds for $y\in \pmm$. Based on this bound, the PAC bound $8\Popt$ is derived \citep{linial1993constant}. Hence, this raises the question as to whether taking the $\sign$ is optimal in $\mathcal{L}_2$-based PAC learning. When $x\in \pmm^d$, Blum \etal~ \citep{Blum1994} and Jackson \citep{Jackson2006} proposed a clever idea of randomized rounding instead of taking the sign. As a result, they improved the factor from $8\Popt$ to $2\Popt$. In Section \ref{sec:sign}, we argue that these bounds are loose, at least for binary inputs. We prove new bounds connecting the 0-1 loss and the square loss (Lemma \ref{lem:norm2 sign} for binary input and Lemma \ref{lem:Pe of sign p-theta} for real-valued inputs). Using these results, we show that for $k$-junta class, taking the sign is not problematic and gives $\Popt$, hence $\mathcal{L}_2$-based agnostic PAC learnability. Moreover, we improve the factor $8\Popt$ to $2\Popt$ for more general classes with $x\in \RR^d$.

Our approach relies on a framework using vector spaces equipped with probability measures as a proxy to derive PAC learning bounds. Among others, we consider a joint vector space for functions on the feature-label set $\mathcal{X}\times\mathcal{Y}$, incorporating the sample-label relation and the underlying joint distribution $D$. 
 This approach establishes our results by connecting the PAC learning model and powerful tools for analyzing vector spaces. 
   Notably, we prove an elegant upper bound on the 0-1 loss based on amenable quantities such as $1$-norm and $2$-norm (see Corollary \ref{cor:kjunta_nonbinary} and \ref{lem:norm2 sign} in Section \ref{sec:analysis}).  A notable feature of our approach is that the expressions are quite compact and insightful. 


\begin{table*}[ht]
\centering
\caption{Comparison of the PAC-learning algorithms for $k$-juntas.}
\label{tab:PAC compare}
\begin{tabular}{lp{2.5cm}p{2.5cm}p{6cm}}
\toprule 
Algorithm & Sample Cmplx. & Comp. Cmplx. & PAC Error \\ 
\hline \\[-1em]
\begin{minipage}{2.8cm} Brute force ERM\end{minipage} & { $O(\frac{k2^k}{\epsilon^2}\log\frac{d}{\delta})$}  & $O(nd^k2^{2^k})$ & $\Pek+\epsilon$ \\ [3pt]
\hline \\[-1em]
\begin{minipage}{2.8cm}$\mathcal{L}_1$-Poly. Reg. \\\citep{Kalai2008}\end{minipage} & {$O(\frac{1}{\epsilon}k^{\Theta(k)}\log \frac{d}{\delta})$} &  $O(n^2d^{(3+\omega)3k})$ & $\Pek+\epsilon$ \\ [3pt]
\hline \\[-1em]
\begin{minipage}{2.8cm} $\mathcal{L}_2$-Poly. Reg.  \end{minipage}& 
$O(\frac{k2^k}{\epsilon^2}\log\frac{d}{\epsilon^2 \delta})$ & $O(n d^{\Theta(k)})$ &  
\begin{minipage}{4cm}
$\bullet~2\Pek+\epsilon$ \citep{Jackson2006}\\
 $\bullet~ \Pek + \epsilon$, [{\bf \color{red}{Thm.} \ref{tab:PAC compare}}]
\end{minipage}\\[3pt]
\hline \\[-1em]
\begin{minipage}{2.8cm} Low-degree Alg.\\ (\textit{uniform dist.})\\ \citep{linial1993constant} \end{minipage}& $O(k2^k \log\frac{d}{\delta})$ & $O(\frac{nkd^k}{(k-1)!})$ & \begin{minipage}{6cm} $\bullet~ 8\Pek+\epsilon$  \citep{linial1993constant}\\
 $\bullet~2\Popt+\epsilon$ \citep{Jackson2006}\\
 $\bullet~\frac{1}{4}+\Popt(1-\Popt)+\epsilon$ \citep{Kearns1994}\\
 $\bullet~ \Pek + \epsilon$, [{\bf \color{red}{Thm. \ref{thm:Fourier PAC}}}]
 \end{minipage}   \\ [3pt]
\hline \\[-1em]
\begin{minipage}{2.8cm} Stochastic Fourier\\({\bf \color{red}{Algorithm \ref{alg:fourier}}}) \end{minipage}& $O(\frac{k 2^{k}}{\epsilon^2}\log \frac{d}{\delta})$ & $O(\frac{nkd^k}{(k-1)!})$ &  $\Pek+\epsilon$, [{\bf \color{red}{Thm. \ref{thm:Fourier PAC}}}]\\ 
\bottomrule
\end{tabular} 
\end{table*}

\noindent\textbf{Learning with Fourier algorithm:} In addition, we present another more efficient algorithm for binary-valued samples. This algorithm's running time is linear in $n$ and  scales with  $d^k$ which is asymptotically better than the two other approaches as they grow with $d^{O(k)}$. Our result relies on the Boolean Fourier expansion defined for the uniform distribution \citep{Wolf2008,ODonnell2014}. We prove a counter-intuitive result by showing that the uniform Boolean Fourier is in fact applicable to agnostic distribution-free settings.  Motivated by Linial's low-degree algorithm  \citep{linial1993constant} on uniform distribution, we develop a Fourier algorithm that performs $\mathcal{L}_2$ polynomial regression more efficiently and without any distributional assumption. 
We then show that this algorithm also agnostically PAC learns the $k$-junta class. More formally, we prove the following statement.
\begin{customthm}{2 (abbreviated)}
Given $k<d$, the Fourier algorithm (Algorithm \ref{alg:fourier}) agnostically PAC-learns $k$-juntas  with sample complexity $O(\frac{k 2^{k}}{\epsilon^2}\log \frac{d}{\delta})$ and computational complexity $O(\frac{nkd^k}{(k-1)!})$. 
\end{customthm} 

Table \ref{tab:PAC compare} compares  various PAC learning algorithms in terms of their sample complexity, running time, and PAC loss.   The $\mathcal{L}_2$ and Fourier algorithms have lower sample and computational complexities when compared to the other methods.  
When compared to \citep{Kalai2008} using $\mathcal{L}_1$-polynomial regression, we obtain a lower sample complexity and computational complexity. 
Note that the running time of $\mathcal{L}_1$ regression grows with $O(n^2d^{O(k)})$, which is quadratic in sample size $n$ and hence prohibitive in large data sets. The running time of $\mathcal{L}_2$ regression is $O(nd^{O(k)})$ which is linear in $n$. Lastly, the running time of the Fourier algorithm grows with $O(nd^k)$, which is a better exponent than $\mathcal{L}_2$. Overall, given that $k$ is typically a constant independent of $d$, the $\mathcal{L}_2$ regression and the Fourier algorithm are suitable for large data sets. Lastly, we present a lower bound on the sample complexity of the $k$-junta class.  Based on the standard VC-dimension argument which gives $O(\frac{1}{\epsilon^2}(VC+\log (\frac{1}{\delta}))$. The exact expression for the VC dimension of the $k$-junta class is unknown, but it is  between $2^k$ and $2^k+k\log d$.

 \subsection{Related Works} \label{subsec:related approaches}
 The problem of learning juntas is a classical problem in machine learning. 
 There is a large body of work on learning and testing of juntas  \citep{mossel2004learning,Bshouty2016,Liu2019,Arpe2008,Fischer2004,Servedio2015,De2019,Vempala2011,Chen2021,Iyer2021}. Juntas are of significant interest in learning theory as they are connected to other fundamental problems such as learning with feature selection \citep{Guyon2003},  DNF formulas, and decision trees \citep{mossel2004learning}. Particularly, learning with feature selection can be expressed as learning $k$-juntas (with $k$ out of $d$ features). Additionally, every $k$-junta is implemented by a decision tree or DNF formula of size $2^k$ and conversely, any size-$k$ decision tree is also a $k$-junta, and any $k$-term DNF is $\epsilon$-approximated by a $k \log(\frac{k}{\epsilon})$-junta. Hence,  obtaining efficient algorithms for these problems is closely related to learning juntas \citep{mossel2004learning}. PAC learning with respect to $k$-juntas has been studied using various approaches. We briefly review the approaches for learning these concept classes below and summarize them in Table~\ref{tab:PAC compare}.\\
\noindent\textbf{Naive  \ac{ERM}:}
This is  the usual exhaustive search over all predictors to minimize the empirical loss. For $k$-juntas, ERM is an agnostic PAC learning algorithm with  sample complexity $O(\frac{k2^k}{\epsilon^2}\log\frac{d}{\delta})$ and computational complexity $O(nd^k2^{2^k})$ \citep{Shalev2014}. With the computational complexity of doubly exponential with respect to $k$, ERM is prohibitive even for small values of $k$.\\
\noindent\textbf{Learning with $\mathcal{L}_1$ Regression.}
  Kalai \etal~ \citep{Kalai2008} introduced polynomial regression as an approach for PAC learning with the $0-1$ loss function. They showed that $\mathcal{L}_1$-Polynomial regression agnostically PAC learns with respect to $(k,\epsilon)$-concentrated hypothesis class which includes $k$-juntas. 
Adopting this algorithm to $k$-juntas requires a sample complexity $O(\frac{1}{\epsilon}d^{\Theta(k)})$. With a \textit{linear programming} implementation, the computational complexity of this algorithm is $O(n^2d^{(3+\omega)3k})$, where $\omega < 2.4$ is the matrix-multiplication exponent. The quadratic growth of the computational complexity of this approach makes it expensive for large sample sizes. This motivates us to study $\mathcal{L}_2$ based approaches. \\ 
\noindent\textbf{Learning with $\mathcal{L}_2$ Polynomial Regression.}
This approach is similar to its $\mathcal{L}_1$ counterpart with absolute error replaced by the square loss. Fast implementations of $\mathcal{L}_2$ regression with linear complexity in sample size have been studied \citep{Drineas2006,Drineas2010}. PAC learning using this approach has been studied in \citep{Kalai2008,Jackson2006}. In the agnostic setting, it is shown that this approach is a \textit{weak learner} with error $8\Popt$. With the use of a nondeterministic rounding proposed in \citep{Blum1994,Jackson2006}, the PAC bound can be reduced to $2\Popt$. This paper shows that for $k$-juntas $\Popt$ is obtained without any randomized rounding. For other non-binary classes, in Section \ref{sec:other classes}, we prove the bound $2\Popt$.\\
\noindent\textbf{Fourier Algorithms.}
This approach is viewed as a special solution for $\Ltwo$ regression. Linial \etal~ \citep{linial1993constant} investigated PAC learning from an alternative perspective and introduced the well-known ``Low-Degree Algorithm". They provide theoretical guarantees under the {\it uniform} and {\it known} distribution on $\pmm^d$ of the samples. 
The low-degree is based on the Fourier expansion on the Boolean cube. Although  computationally efficient, this algorithm has limited practical applications due to its distributional restrictions --- uniform (and known) distribution is unrealistic in many applications.  Furst \etal~ \citep{furst1991improved} relaxed such a distributional restriction by adopting a low-degree algorithm for learning $AC^0$ functions under the product probability distributions. The Fourier expansion has been  used to analyze Boolean functions  \citep{Wolf2008,ODonnell2014} with a wide range of applications, namely computational learning \citep{linial1993constant,mossel2004learning}, noise sensitivity \citep{ODonnell2014,kalai2005noise,LiMedardISIT18,Heidari_ISIT19}, approximation \citep{blais2010polynomial}, feature selection \citep{HeidariICML2021}, and other information-theoretic problems \citep{courtade2014boolean,weinberger2017optimal,weinberger2018self,Heidari2021ISIT}. 
In this work, we also generalize this approach for \textit{agnostic} PAC learning --- hence, removing the distributional assumptions.


%% file: FourierOrth.tex
\subsection{Fourier-Based Learning Algorithm}\label{sub:Fourier PAC concentrated}
We present another $\Ltwo$-based approach that is computationally more efficient than the $\Ltwo$-polynomial regression. The computational cost of the $\Ltwo$ regression grows as $O(nd^{\Theta(k)})$ which is more efficient than its $\mathcal{L}_1$ variant with complexity $O(n^2d^{(3+\omega)3k})$. This leads to the question as to whether the factor $d^{\Theta(k)}$ can be further reduced. We answer this question using a Fourier analysis on the Boolean cube. Particularly, we present an algorithm with the complexity of $O(\frac{nkd^k}{(k-1)!})$.

Our solution is based on the Boolean Fourier expansion applied to the uniform distribution on the Boolean cube \citep{ODonnell2014,Wolf2008}. Surprisingly, we plan to use this Fourier for agnostic settings.  
%
Let us briefly explain the standard Boolean Fourier expansion. 

\begin{fact}[Boolean Fourier]\label{fact:Fourier}
Any (bounded) function $f: \pmm^d\rightarrow \RR$ admits the following decomposition 
\begin{equation*}
 f(\bfx) = \sum_{\mathcal{S}\subseteq [d]} \mathsf{f}_{\mathcal{S}}~\chi_{\mathcal{S}}(\bfx), \quad \forall \bfx\in \pmm^d,
 \end{equation*}
where $ \pS(\bfx)$ is the monomial corresponding to the subset $\mathcal{S}\subseteq [d]$ and is defined as $ \pS(\bfx)= \prod_{j\in \mathcal{S}}x_j$. Further, the coefficients $ \mathsf{f}_{\mathcal{S}}\in \RR$ are called the Fourier coefficients of $f$ and are calculated as $$ \mathsf{f}_{\mathcal{S}} = \frac{1}{2^d}\sum_{\bfx} f(\bfx) \chi_{\mathcal{S}}(\bfx), \quad \forall \mathcal{S}\in [d]$$ 
\end{fact}
 This expansion relies on the restriction that the input variables are uniformly distributed over the Boolean cube. This limits the applications of Fourier-based algorithms such as \citep{linial1993constant} to \textit{agnostic} learning problems without any distributional assumptions.   
 This issue can be resolved via a Gram-Schmidt-type orthogonalization process that yields a generalized Boolean Fourier expansion \citep{Heidari2021ISIT}.
 
 However, in this paper, we take a slightly different path and propose a simple adjustment to the standard Boolean Fourier that applies to certain agnostic problems. Hence, we get PAC learnability together with computational efficiency. In what follows, we describe this adjustment. 

Let $D_X$ be any probability distribution on $\pmm^d$ and $f$ be a Boolean function. Define 
\begin{align*}
\fS:= \frac{1}{2^d}\sum_{\bfx} f(\bfx) D_X(\bfx) \chi_{\mathcal{S}}(\bfx).
\end{align*}
From Fact \ref{fact:Fourier}, $\fS$ is the Fourier coefficient of the real-valued function $f(\bfx)D_X(\bfx)$. Note that under the uniform $D_X$, $\fS=\frac{1}{2^d} \mathsf{f}_{\mathcal{S}}$, where $\mathsf{f}_{\mathcal{S}}$ is the Fourier coefficient of $f(\bfx)$ as in Fact \ref{fact:Fourier}. In agnostic settings where $D_X$ is unknown, $\fS$ is not accessible. However, we can estimate it empirically. 


Before explaining the estimation, let us introduce another extension. In agnostic settings, the label $y$ is not necessarily a function of the features $\bfx$. Hence, to make the Fourier expansion applicable to agnostic PAC, we expand it, beyond deterministic function, to stochastic mappings:

  Consider a random vector $\bfX$ and a labeling variable $Y$. Let $(\bfX, Y)\sim D$ where $D$ is a probability distribution over $\pmm^d\times \pmm$. Then the stochastic Fourier coefficients are defined as
\begin{align}\label{eq:Y Fourier}
\aS:= \frac{1}{2^d}\EE[Y \pS(\bfX)], 
\end{align}
for all $\mathcal{S}\subseteq [d]$. If $Y=f(\bfX)$, then $\aS=\fS$. Given the \ac{IID} samples  $\set{x(i), y(i)}_{i=1}^n$, the empirical estimation of $\aS$ is
\begin{align}\label{eq:Y Fourier empirical}
\ahatS :=\frac{1}{2^d} \frac{1}{n}\sum_{i=1}^n  y(i) \chi_{\mathcal{S}}(\bfx(i)).
\end{align}
Note that the estimation is agnostic to the underlying distribution $D$, but we show that it converges to $\aS$.
\begin{lemma}\label{lem:Fourier estimation err}
Let $D$ be any probability distribution on $\pmm^{d+1}$. Let $\mathcal{S}_1, \mathcal{S}_2, \cdots, \mathcal{S}_m$ be $m$ subsets of $[d]$. Given $\delta\in (0,1)$ and $n$  samples drawn \ac{IID} from $D$, the  inequality 
\begin{align*}
\sup_{1\leq j\leq m} |\hat{a}_{\mathcal{S}_j} -a_{\mathcal{S}_j}|\leq \frac{1}{2^d}\sqrt{\frac{1}{2n}\log \frac{2m}{\delta}}
\end{align*}
holds with probability at least $(1-\delta)$, where $\aS$ and $\ahatS$ are given in \eqref{eq:Y Fourier} and \eqref{eq:Y Fourier empirical}, respectively.
\end{lemma}
\begin{proof}
Observe  that for any $\mathcal{S}\subseteq [d]$
\begin{align*}
\EE_{\mathcal{S}_n\sim D^n}[\ahatS] &= \frac{1}{2^d} \EE[Y(1) \chi_{\mathcal{S}}(\bfX(1))] \\
 &=\frac{1}{2^d} \sum_{\bfx, y} D(\bfx, y)  y~\chi_{\mathcal{S}}(\bfx)=\aS.
\end{align*}
By taking the factor $\frac{1}{2^d}$ in the definition of $\aS$ and $\ahatS$, we have that
\begin{align*}
|\ahatS -\aS| = \frac{1}{2^d}\Big| \frac{1}{n}\sum_{i=1}^n  y(i) \chi_{\mathcal{S}}(\bfx(i)) -\EE[Y\pS(X)\Big|. 
\end{align*}
We apply McDiarmid inequality to bound the right-hand side term. It is not difficult to check that 
\begin{align*}
\prob{ \Big| \frac{1}{n}\sum_{i=1}^n  y(i) \chi_{\mathcal{S}}(\bfx(i)) -\EE[Y\pS(X)\Big| \geq \epsilon} \leq 2e^{-\frac{n\epsilon^2}{2}}.
\end{align*}
Therefore, by considering the factor $\frac{1}{2^d}$,  from the union bound, and by  equating the right-hand side to $\delta$, we establish the lemma. 
\end{proof}
 With this approach, our Fourier algorithm (See Algorithm \ref{alg:fourier}) performs a polynomial regression in the Fourier domain by estimating the Fourier coefficients of the label from the training samples. In the following theorem, we present a PAC bound for learning $k$-juntas using this approach.

  \begin{algorithm}[t!]
\caption{Stochastic Fourier}
\label{alg:fourier}
\DontPrintSemicolon
\KwIn{Training samples $\mathcal{S}_n=\{(\bfx(i), y(i))\}_{i=1}^n$, degree parameter $k$.}
\KwOut{Predictor $\predict$}
\SetKwProg{Fn}{}{:}{} 
\SetKwFunction{main}{LearningAlgorithm}

For each $\mathcal{S}\subseteq [d]$ with at most $k$ elements 
 compute the empirical Fourier coefficients as 
$ \ahatS = \frac{1}{2^d}\frac{1}{n}\sum_{i=1}^n y(i) \prod_{j\in \mathcal{S}}x_j(i)$.\;
For each $\mathcal{J}\subseteq [d]$ with $k$ elements 
 construct the function $\festJ(\bfx) =\sum_{\mathcal{S}\subseteq \mathcal{J} }    \ahatS \pS(\bfx).$\; 
Find $\Jalg$ with the minimum empirical loss of $\sign[\festJ]$. \;
\KwRet  $\predict \equiv\sign[\fJalg]$.\;
\end{algorithm}

\begin{theorem}\label{thm:Fourier PAC}
The Fourier algorithm agnostically learns $k$-juntas for $k\leq d/2$ and with error less than
\begin{align*}
\Pek + O\Big(\sqrt{\frac{  2^k}{n }\log \frac{d^k}{(k-1)!\delta}}\Big),
\end{align*}
with probability at least $(1-\delta)$, where $n$ is the number of samples. Moreover, the resulted computational complexity is $O(\frac{nkd^k}{(k-1)!})$.  
\end{theorem}

In the next section, we discuss our main ideas. The proof of this theorem is given in Section \ref{subse:proof:thm:Fourier PAC}.

%% file: analysis.tex
\section{Main Technical Results}\label{sec:analysis}
The main results of this paper rely on a fundamental connection between square loss and the 0-1 loss presented as Corollary \ref{cor:kjunta_nonbinary} and \ref{lem:norm2 sign} in Section \ref{sub:technical analysis}. In this section, we present this connection and describe the steps in proving Theorem \ref{thm:L2 PAC} and \ref{thm:Fourier PAC}.

\subsection{A Vector Space Representation}\label{sub:HilbertSP}
We introduce a vector representation incorporating the feature-label  distribution. Such representation is a proxy to use powerful algebraic tools developed for vector spaces.  In what follows, we describe this representation. 

Let $\mathcal{X}$ denote the input set, $\mathcal{Y}=\pmm$ be the label set, and $D$ be the underlying  distribution on  $\mathcal{X}\times \mathcal{Y}$. 
Consider the vector space of all functions  $h:\mathcal{X}\times \mathcal{Y} \mapsto \RR$ for which $\EE_D[h(\bfX,Y)^2]$ is finite\footnote{A zero function in this space is a function  that maps $(\bfx, y) \mapsto 0$ for all $\bfx, y$ except a  zero-probability subset.}. Naturally, the inner product between two functions $h_1$, and $h_2$ is defined as $$\<h_1, h_2\>_D\deq \EE_D[h_1(\bfX,Y) h_2(\bfX, Y)].$$

With this formulation, the true labeling is simply the function $(\bfx, y)\mapsto y$. Note that this complies with the agnostic setting, where the label is not necessarily a function of $\bfx$. In addition,  a predictor $g$ in the learning model is viewed as the mapping $(\bfx, y)\mapsto g(\bfx)$.  
Since $\mathcal{Y}=\pmm$, then the 0-1 loss of any predictor $g$ can be written as 
\begin{align}\label{eq:error and inner prod}
\Loss(g) &= \frac{1}{2}-\frac{1}{2}\<Y,g\>_D = \frac{1}{4}\norm{Y-g}_{2,D}^2,
\end{align}
where, with slight abuse of notation, $Y$ and $g$ are understood as the mappings $(\bfx,y)\mapsto y$ and $(\bfx,y)\mapsto g(\bfx)$, respectively. This first equality in \eqref{eq:error and inner prod} is because of the identity $\11\set{a\neq b}=\frac{1}{2}(1- ab)$ for any $a,b\in \pmm$. The second equality is from the definition of 2-norm and the fact that $\norm{Y}_{2,D}=\norm{g}_{2, D}=1$.

One benefit of this representation is that the theoretical results under the known distribution $D$ can be easily translated to the agnostic setting. This is easily done  by replacing $D$ with the empirical distribution $\Demp$ that is uniform on the training set and zero outside of it.  For instance, the empirical loss of  $g$ immediately satisfies the same type of relationship as in \eqref{eq:error and inner prod}:
\begin{equation*}
 \frac{1}{n}\sum_i \11\set{y_i\neq g(\bfx_i)} 	=\frac{1}{2}-\frac{1}{2}\<Y, g\>_{\Demp}=\frac{1}{4}\norm{Y-g}_{2,\Demp}^2.
\end{equation*}

\subsection{PAC and MMSE }\label{sub:technical analysis}
In what follows, we derive bounds on the expected and empirical loss and prove the main theorems. 
The main ingredient in the proof of the main results (Theorem \ref{thm:L2 PAC} and \ref{thm:Fourier PAC})  is a connection between the \ac{MMSE} and the PAC learning loss.

Consider a general problem in which $Z$ is the observations and the goal is to predict $Y$. Here $Z$ takes values from a generic set $\mathcal{Z}$ and $Y$ from $\pmm$. Let $Y_{MMSE}$ be the \ac{MMSE} of $Y$ given $Z$. It is known that $Y_{MMSE} = \EE[Y|Z]$. In the following lemmas, we establish the connection between MMSE and PAC.
 The proofs are provided in Appendix \ref{proof:lem:MMSE and 01} and \ref{app:proof norm2 sign}. 
\begin{lemma}\label{lem:MMSE and 01}
Suppose $(Y, Z)\sim D$ is a pair of random variables, where $Y$ takes values from $\pmm$ and $Z$ from some set $\mathcal{Z}$. Suppose $g:\mathcal{Z}\rightarrow \pmm$ is any predictor of $Y$ from $Z$. Then, 
\begin{align*}
\prob{Y\neq g(Z)} = \frac{1}{2}-\frac{1}{2} \<{Y}_{MMSE}, g\>.
\end{align*}
Moreover, let $opt_Z$ be the minimum 0-1 loss among all predictors of $Y$ given $Z$. Then,
\begin{equation}\label{eq:minimum_mismatch_prob general}
opt_Z=\frac{1}{2}-\frac{1}{2}\EE\Big[\big|\EE[Y|Z]\big| \Big].
\end{equation}
Lastly, $g^*\equiv\sign[{Y}_{MMSE}]$ is the optimal predictor.
\end{lemma}

\begin{lemma}\label{lem:norm2 sign}
Let $\mathcal{Z}$ be any set and $h:\mathcal{Z}\rightarrow \RR$ be any bounded function. Suppose $(Y, Z)\sim D$ be a pair of random variables, where $Y$ take values from $\pmm$ and $Z$ from $\mathcal{Z}$. Then,  
\begin{align*}
\prob{Y \neq \sign[h(Z)]}&\leq opt_Z+ U\big(\EE\Big[\big({Y}_{MMSE}-h(Z)\big)^2\Big]\big),
\end{align*}
where $U$ is a polynomial defined as $U(x)=x^3+\frac{3}{2}x^2+\frac{3}{2}x$.
\end{lemma}

\noindent\textbf{Connections to learning $k$-juntas:} Given the above results, we can derive bounds on the error in learning many classes such as $k$-juntas. Let $\mathcal{J}$ be a subset of $[d]$ with $k$ elements. Set $Z=X^{\mathcal{J}}$ as our observation variable.  Consider all polynomials  on the coordinates of $\mathcal{J}$ as the input variables.  The polynomial that minimizes the square loss is defined as the projection of $Y$ onto the subset $\mathcal{J}$. This polynomial is formally defined as 
\begin{align}\label{eq:Pi_YJ}
\PiJ := \argmin_{p\in \Pk} \norm{Y-p(X^{\mathcal{J}})}_{2,D}
\end{align}
where $\Lopt$ is the set of polynomials of degree at most $k$. Note that $\PiJ $ is the \ac{MMSE} of $Y$ from the observation $Z=X^{\mathcal{J}}$. Then we immediately get the following result from Lemma \ref{lem:MMSE and 01}.
\begin{corollary}\label{cor:kjunta_nonbinary}
Let $\Lopt$ be the minimum 0-1 among all the $k$-juntas for a fixed $k \leq d$. Then, 
\begin{equation}\label{eq:minimum_mismatch_prob}
\Pek=\frac{1}{2}-\frac{1}{2}\max_{\mathcal{J}\subseteq [d], \,  
\abs{\mathcal{J}} = k}\norm{\PiJ}_{1,D}.
\end{equation}
\end{corollary}
Based on these results, we are ready to prove Theorem \ref{thm:L2 PAC} on PAC learning of $k$-juntas using $\mathcal{L}_2$ regression. 
\subsection{Proof of Theorem \ref{thm:L2 PAC}}\label{subsec:proof:thm:L2 PAC}

For any $\mathcal{J}$, let $\hat{p}_{\mathcal{J}}$ be the output of the empirical polynomial regression, that is $\hat{p}_{\mathcal{J}}=\argmin_{p\in\Pk} \norm{Y-p_{\mathcal{J}}}_{2, \Demp}$, where $\Demp$ is the empirical distribution. Note that the selected predictor is of the form $\sign[\hat{p}_{\mathcal{J}}]$, as in Algorithm \ref{alg:L2}. As a result, from Corollary \ref{lem:norm2 sign} with $D$ replaced with $\Demp$ and $Z=\bfX^{\mathcal{J}}$, the empirical loss of $\sign[\hat{p}_{\mathcal{J}}]$ is bounded as 
$\Lossemp(\sign[\hat{p}_{\mathcal{J}}]) \leq  \frac{1}{2}-\frac{1}{2}\norm{\hat{p}_{\mathcal{J}}}_{1, \Demp},$
\addvb{where the $U(\cdot)$ term in Lemma \ref{lem:norm2 sign} is zero, as $\hat{p}_{\mathcal{J}}$ is the MMSE of $Y$ under $\hat{D}$. }
Next,  we minimize both sides over all $k$-element subsets $\mathcal{J}$. From  Corollary \ref{cor:kjunta_nonbinary}, with $D$ replaced by $\Demp$, the right-hand side of the above inequality minimized over $\mathcal{J}$ is the minimum empirical loss $\Pekemp$. This implies that
$ \min_{\mathcal{J}: |\mathcal{J}| = k} \Lossemp(\sign[\hat{p}_{\mathcal{J}}])= \Pekemp.$
Hence, we proved that the minimum empirical loss is achieved using the $\mathcal{L}_2$ polynomial regression. Naturally, the next step is to extend this result to the generalization loss. This part follows from the standard arguments in VC theory ( See Corollary 3.19 in \citep{Mohri2018}) and the fact that the VC dimension of the $k$-junta class is less than $2^k+O(k\log d)$. Particularly, given $\delta\in (0,1)$, with probability $(1-\delta)$, the generalization loss is less than
$\Pek + O\big(\sqrt{\frac{2^k+ k\log d}{n}\log \frac{n}{2^k+ k\log d}}\big)+ \sqrt{\frac{\log(1/\delta)}{2n}},$
where $n$ is the number of samples. With this inequality, the theorem is proved.

%% file: Fourier_Proof.tex
 \subsection{PAC Learning in Fourier Domain}
 Next, we analyze the Fourier algorithm and prove Theorem \ref{thm:Fourier PAC}. We study the PAC learning problem in the Fourier domain. For that, we start with the following lemma connecting the prediction loss to the Fourier coefficients.  
 \begin{lemma}\label{lem:Fourier and loss}
 Let $(\bfX, Y)\sim D$ where $D$ is a distribution on $\pmm^{d+1}$. Then the prediction loss of any $g(\bfx)$ equals to
 \begin{align*}
 \Loss(g) &= \frac{1}{2}-2^{d-1}\sum_{\mathcal{S}\subseteq [d]} \aS \gS,
 \end{align*}
 where $\gS$ is the (uniform) Fourier coefficient of $g$ corresponding to $\mathcal{S}$, as in Fact \ref{fact:Fourier}, and $\aS$ is the stochastic Fourier coefficient of $Y$ as in \eqref{eq:Y Fourier}.  
 \end{lemma}
\begin{proof}
Recall from \eqref{eq:error and inner prod} that $\Loss(g)= \frac{1}{2}-\frac{1}{2}\EE[Yg(\bfX)]$.  Then, from the definition of $\aS$ in \eqref{eq:Y Fourier}, we have that
 \begin{align*}
\EE[Yg(\bfX)] &= \sum_{y, \bfx} D(x,y) y g(\bfx)\\
  &=\sum_{y, \bfx} y D(x,y)  \big(\sum_{\mathcal{S}} \gS \pS(\bfx)\big)\\
    &=\sum_{\mathcal{S}}\gS \sum_{y, \bfx} y D(x,y)  \pS(\bfx)\\
    &=\sum_{\mathcal{S}\subseteq [d]}  \gS (2^d\aS),
 \end{align*} 
 as needed.
 \end{proof}
 Interestingly, with this lemma, the prediction loss under any distribution $D$ can be written in terms of $\gS$'s which are the Fourier coefficient of $g$ under the uniform distribution. We use this intuition and prove the following lemma in Appendix \ref{proof:lem:Fourier framework}.

 \begin{lemma}\label{lem:Fourier framework}
  Let $(\bfX, Y)\sim D$ where $D$ is a distribution on $\pmm^{d+1}$. Given any subset coordinate $\mathcal{J}$, let  $\fJ(\bfx) = \sum_{\mathcal{S}\subseteq \mathcal{J}} \aS \pS(\bfx)$, with $\aS$'s being the stochastic Fourier coefficients of $Y$. Let $\hJ$ be any real-valued function on coordinate $\mathcal{J}$, then the prediction loss of $g\equiv \sign[\hJ]$ is bounded as
 \begin{align*}
\loss(g) &\leq \frac{1}{2}(1-\norm{\fJ}_{1, \unif}) + U(\norm{\fJ-\hJ}_{2, \unif}),
 \end{align*}
  where the norm is computed  on the uniform distribution and $U(x)=x^3+\frac{3}{2}x^2+\frac{3}{2}x$.
  \end{lemma}
  This lemma is different from Lemma \ref{lem:norm2 sign} in that $\fJ$ is not  the MMSE estimate of $Y$ as it is  defined based on the uniform Fourier expansion.
   However, it gives a different characterization of the optimal loss $\Pek$. 
  \begin{corollary}\label{cor:popt Fourier}
The optimal loss among $k$-juntas under any distribution $D$ satisfies the following equation 
   \begin{align*}
 \Pek = \frac{1}{2}-\frac{1}{2}\max_{\mathcal{J}\subseteq [d], \,  
\abs{\mathcal{J}} = k}\norm{\fJ}_{1,\unif}.
 \end{align*}
  \end{corollary}
  Based on these results, we prove Theorem \ref{thm:Fourier PAC} on the PAC learning of the Fourier algorithm.
 \subsection{Proof of Theorem \ref{thm:Fourier PAC} }\label{subse:proof:thm:Fourier PAC}
We prove the theorem by showing that $\predict$ in Algorithm \ref{alg:fourier} achieves $\Pek$ of $k$-juntas. Recall that $\predict\equiv\sign[\fJalg]$, where  $\fJalg$ is the constructed for the selected subset $\Jalg$. Thus, from Lemma \ref{lem:Fourier framework}, the prediction loss of $\predict$ is bounded as
\begin{align*}
\loss(\predict) &\leq \frac{1}{2}(1-\norm{f^{\Jalg}}_{1, \unif}) +U(\norm{f^{\Jalg}-\fJalg}_{2, \unif}).
\end{align*}
Next, we bound the second term on the right-hand side. Note that $\fJalg\equiv \sum_{\mathcal{S}\subseteq \Jalg} \ahatS \pS$. 
Parseval identity gives 
\begin{align}\label{eq:pstar-Piest}
 \norm{f^{\Jalg}- \fJalg }^2_{2,\unif} &= \sum_{\mathcal{S}\subseteq \Jalg} (\aS-\ahatS)^2.
 \end{align}
 Consider all $\mathcal{S}\subseteq [d]$ with at most $k$ elements. Let $K$ be the number of such subsets. Then, as $|\Jalg|=k$, using  Lemma \ref{lem:Fourier estimation err} the above summation is bounded as,  
 \begin{align*}
  \norm{\fJ- \festJ }^2_{2,\unif} \leq 2^{k} \sup_{\mathcal{S}: |\mathcal{S}|\leq k} (\aS-\ahatS)^2 \leq  {\frac{2^k}{2n }\log \frac{2K}{\delta}},
 \end{align*}
 where the second inequality holds with probability at least $(1-\delta)$.
As a result, the prediction loss satisfies 
\begin{align*}
\loss(\predict) \leq \frac{1}{2}(1-\norm{f^{\Jalg}}_{1,\unif})+ O\Big(\sqrt{\frac{2^k}{n }\log \frac{K}{\delta}}\Big),
\end{align*}
where we used the fact that $U(x)\leq 4x$ for $x\in [0,1]$.   Next, we minimize the right-hand side over the choice of $\Jalg$ by considering all $k$-element coordinates $\mathcal{J}$. Let $\mathcal{J}^*$ be the optimal set. Then, from Corollary \ref{cor:popt Fourier}, we obtain that 
\begin{align*}
\loss(\sign[\hat{f}^{\mathcal{J}^*}]) &\leq \Pek + O\Big(\sqrt{\frac{2^k}{n }\log \frac{K}{\delta}}\Big).
\end{align*}
Note that $\mathcal{J}^*$ is not necessarily the same as the algorithm's choice $\Jalg$.
However, as $\Jalg$ is the $k$-element coordinate that minimizes the empirical loss, then $\Lossemp(\sign[\fJalg]) \leq \Lossemp(\sign[\hat{f}^{\mathcal{J}^*}])$. Therefore, from McDiarmid's inequality with probability $(1-\delta)$  we obtain that
\begin{align*}
\Lossemp(\sign[\hat{f}^{\mathcal{J}^*}])\leq \Loss(\sign[\hat{f}^{\mathcal{J}^*}]) + \sqrt{\frac{k}{2n}\log\frac{2}{\delta}}, 
\end{align*}
where we used the fact that there are at most $2^k$ Boolean functions on coordinate $\mathcal{J}^*$. To sum up, we proved that 
\begin{align*}
\Lossemp(\sign[\fJalg]) \leq \Pek +  O\Big(\sqrt{\frac{2^k}{n }\log \frac{K}{\delta}}\Big) .
\end{align*}
Assuming that $k\leq d/2$, we bound $K$ as $$K \leq \sum_{\ell=0}^k {d \choose \ell}\leq 1+ k {d \choose k}  = 1+ \frac{d^{k}}{(k-1)!}.$$
The rest of the argument follows from VC theory for replacing $\hat{D}$ with $D$ in the left-hand side.

The computational complexity of Algorithm \ref{alg:fourier} is dominated by the procedure for estimating all the $K$ Fourier coefficients. Each estimation takes $O(nk)$. Hence, the overall computational complexity of the algorithm is $O(nkK)=O(nk\frac{d^k}{(k-1)!})$ as given in Table \ref{tab:PAC compare}. 

%% file: generalhypothesis.tex
\section{Learning Other Hypothesis Classes}\label{sec:other classes}
In this section, we study learning  more general concept classes using the vanilla $\mathcal{L}_2$ polynomial regression (see Algorithm \ref{alg:L2 poly}).  
An important concept class is the set of predictors that are approximated by fixed-degree polynomials as studied in \citep{Kalai2008,blais2010polynomial}. 

\begin{algorithm}[t]
\caption{Learning with $\Ltwo$-Polynomial Regression}
\label{alg:L2 poly}
\DontPrintSemicolon
\KwIn{Training samples $\mathcal{S}_n=\{(\bfx(i), y(i))\}_{i=1}^n$, degree parameter $k$.}
 \SetKwProg{Fn}{}{:}{} 
 \SetKwFunction{main}{LearningAlgorithm}

 Find a polynomial $\hat{p}$ of degree up to $k$ that minimizes  $\frac{1}{n}\sum_i \big(y(i)-p(\bfx(i))\big)^2$.\;
 Find $\theta\in [-1,1]$ such that the empirical error of $\sign[\hat{p}(\bfx)-\theta]$ is minimized.\;
 \KwRet   $\hat{g}\equiv \sign[\hat{p}-\theta]$.\;
\end{algorithm}

\noindent{\bf $(\epsilon,k)$-approximated concept class:}
Given $\epsilon\in [0,1]$, $k\in \NN$ and  any probability distribution $D_{\bfX}$ on $\mathcal{X}$, a concept class $\mathcal{C}$ of  functions $c:\RR^d\mapsto \pmm$ is $(\epsilon,k)$-approximated if
$$\sup_{c\in \mathcal{C}}\inf_{p\in \Pk} \EE\big[\big(c(\bfX)-p(\bfX)\big)^2\big]\leq \epsilon^2,$$
where $\Pk$ is the set of all polynomials of degree up to $k$.   

We prove in Appendix \ref{proof:thm:L2 polynomials} that the $\mathcal{L}_2$ polynomial regression learns  the approximated concept class with error up to $2\Popt+\epsilon$. This is an improvement compared to the best known bound $8\Popt$ in \citep{linial1993constant}. 
\begin{theorem}\label{thm:L2 polynomials}
Given $\epsilon>0$ and $k\in \NN$, the  degree $k$ $\Ltwo$ polynomial regression (Algorithm \ref{alg:L2 poly}) learns any $(\epsilon, k)$-approximated concept class, with probability greater than $(1-\delta)$, and  error up to 
 \begin{align*}
2\Pek +3\epsilon +O\Big(\sqrt{\frac{2~d^{k+1}}{n}\log\frac{en}{d^{k+1}}}\Big)+\sqrt{\frac{1}{2n}\log\frac{1}{\delta}}
\end{align*}
where $d$ is the input dimension and $n$ is the sample size. 
\end{theorem}
Note that when changing the inputs from binary to non-binary, the $\mathcal{L}_2$ polynomial regression is not necessarily agnostic PAC learner as the scalar increases to $2\Popt$.

 This result is derived using the following lemma proved in Appendix \ref{proof:lem:Pe of sign p-theta},  eliminating the need for randomized rounding. 
 \begin{lemma}\label{lem:Pe of sign p-theta}
Suppose $\theta$ is a random variable with the probability density function  $f_\theta(t)=1-|t|$, for  $t\in[-1,1]$. Then, the following bound holds for any polynomial $p$ 
\begin{align*}
\EE_\theta\Big[\Lossemp(\sign[p(\bfX)-\theta])\Big] \leq \frac{1}{2}\norm{Y-{p}}_{2, \Demp}^2.
\end{align*}
\end{lemma}

%% file: Proofs/proof_lemmMMSE.tex
\section{Proof of Lemma \ref{lem:MMSE and 01}}\label{proof:lem:MMSE and 01}

\begin{proof}
From \eqref{eq:error and inner prod} in the main text, the generalization error of $g$ can be written as $\frac{1}{2}-\frac{1}{2}\<Y,g\>$. 
This inner product equals to the following 
\begin{align*}
\<Y,g\> &= \EE\big[Y g(Z)\big] = \EE_Z\Big[\EE_{Y|Z}\big[Yg(Z)]~|~Z\big]\Big]\\
&= \EE_Z\Big[\EE_{Y|Z}\big[Y|Z\big] g(Z)\Big].
\end{align*}

Let $Y_{MMSE}=\EE[Y|Z]$. 
 Hence, we obtain that
\begin{align}\label{eq:loss and Ymmse}
\PP\Big\{Y&\neq  g(Z)\Big\} = \frac{1}{2}-\frac{1}{2}\<{Y}_{MMSE},g\>
\end{align}

Note that
\begin{align*}
\prob{Y\neq g(\bfX)}&=\frac{1}{2}-\frac{1}{2}\<{Y}_{MMSE}, g\>_D \geq \frac{1}{2}-\frac{1}{2}\<|{Y}_{MMSE}|, |g|\>_D\geq \frac{1}{2}-\frac{1}{2}\norm{{Y}_{MMSE}}_{1,D},
\end{align*}
where the last inequality follows as $|g(Z)|=1$. Therefore, we get the bound 
$opt_Z \geq \frac{1}{2}-\frac{1}{2}\norm{{Y}_{MMSE}}_{1,D}$.
Hence, we established a lower-bound on $opt_Z$. Next, we show that this bound is achievable.  For that construct a predictor as $g^*=\sign[{Y}_{MMSE}]$. Then, from the above argument, the generalization error of such $g$ equals
\begin{align*}
\prob{Y\neq \sign[{Y}_{MMSE}]} &=\frac{1}{2}-\frac{1}{2}\<{Y}_{MMSE}, \sign[{Y}_{MMSE}]\>_D =  \frac{1}{2}-\frac{1}{2}\norm{{Y}_{MMSE}}_{1,D},
\end{align*}
where the last equality follows due to the identity $\<h, \sign[h]\>= \norm{h}_1$ for any function $h$. 
Therefore, we showed that the lower bound is achievable which implies that 
$opt_Z = \frac{1}{2}-\frac{1}{2}\norm{{Y}_{MMSE}}_{1,D}$ and that $g^*=\sign[{Y}_{MMSE}]$ is the optimal predictor.
\end{proof}

%% file: Proofs/proof_lem_norm2sign.tex
\section{Proof of Lemma \ref{lem:norm2 sign}}\label{app:proof norm2 sign}
\begin{proof}
For shorthand, let $f(z)=\EE[Y|z]$ for any $z\in \mathcal{Z}$. Hence, $f(Z)=Y_{MMSE}$. From  \eqref{eq:loss and Ymmse} in the proof of Lemma \ref{lem:MMSE and 01}, the generalization error of $\sign[h]$ can be written 
 Hence, we obtain that
\begin{align}\nonumber
\PP\Big\{Y&\neq  \sign[
h(Z)]\Big\} = \frac{1}{2}-\frac{1}{2}\<f,\sign[h]\>
\end{align}
Recall that $\norm{f}_{2,D}:= \sqrt{\EE_D[f(X)^2]}$. Hence, $\norm{a-b}_{2, D}^2=\norm{a}_{2, D}^2+\norm{b}_{2, D}^2-2\<a,b\>$. Therefore,
\begin{align*}
\<f,\sign[h]\>&= \frac{1}{2}\big(\norm{f}_{2, D}^2+\norm{\sign[h]}_{2, D}^2-\norm{f-\sign[h]}_{2, D}^2\big)\\
&=\frac{1}{2}\big(\norm{f}_{2, D}^2+1-\norm{f-\sign[h]}_{2, D}^2\big),
\end{align*}
where we used the fact that $|\sign[h]|=1$. As a result, 
\begin{align}\label{eq:pe 2 norm stoch}
\PP\Big\{Y&\neq  \sign[h(Z)]\Big\} = \frac{1}{4}\big(1-\norm{f}_{2, D}^2+\norm{f-\sign[\hJ]}_{2, D}^2\big).
\end{align}

 In what follows, we bound the $\norm{f-\sign[\hJ]}_{2, D}^2$. By adding and subtracting $h$, we have that
\begin{align*}
 \norm{f-\sign[h]}_{2, D}^2  &\stackrel{(a)}{\leq} \Big(\norm{f-h}_{2, D}+\norm{h-\sign[h]}_{2, D}\Big)^2\\\numberthis \label{eq:mismatch_2norm}
 &= \Big(\norm{f-h}_{2, D}^2+\underbrace{\norm{h-\sign[h]}^2_{2, D}}_{\text{(I)}}+2\norm{f-h}_{2, D}\underbrace{\norm{h-\sign[h]}_{2, D}}_{\text{(II)}} \Big),
\end{align*}
where $(a)$  follows from the Minkowski's inequality for $2$-norm. Next, we provide separate bounds for the terms (I) and (II):

\noindent\textbf{Bounding (I):} Note that  $|h-\sign[h]|= |1-|h||.$ Therefore,
\begin{align*}
\text{(I)}=\norm{h-\sign[h]}^2_{2, D}&=\EE\left[(1-|h(Z)|)^2\right]\\\numberthis\label{eq:fbar-sign}
&= 1+\norm{h}_{2, D}^2-2\norm{h}_{1, D}.
\end{align*}

\noindent\textbf{Bounding (II):} From \eqref{eq:fbar-sign}, we have
\begin{align*}
\norm{h-\sign[h]}_{2, D}^2 &= 1+\norm{h}_{2, D}^2-2\norm{h}_{1, D}\\
&\stackrel{(a)}{\leq} 1+ 2( \norm{f}^2_{2, D}+ \norm{f-h}^2_{2, D})-2\norm{h}_{1, D}\\
&\stackrel{(b)}{=} 1+ 2( \norm{f}^2_{2, D}+ \norm{f-h}^2_{2, D})-2\big(\norm{f}_{1, D}+ ( \norm{h}_{1, D}-\norm{f}_{1, D})\big)\\
&= 1+ 2( \norm{f}^2_{2, D}- \norm{f}_{1, D})+ 2 \norm{f-h}^2_{2, D}-2\big( \norm{h}_{1, D}-\norm{f}_{1, D}\big)\\
&\stackrel{(c)}{\leq}  1+ 2 \norm{f-h}^2_{2, D}-2\big( \norm{h}_{1, D}-\norm{f}_{1, D}\big)\\\numberthis \label{eq:temp1}
&\stackrel{(d)}{\leq } 1+ 2 \norm{f-h}^2_{2, D}+2\norm{f-h}_{2, D},
\end{align*}
where $(a)$ follows from the Minkowski's inequality for $2$-norm and the inequality $(x+y)^2\leq 2(x^2+y^2)$. Equality $(b)$ follows by adding and subtracting  $\norm{f}_1$. Inequality $(c)$ holds as $|f(x)|\leq 1$ implying that $\norm{f}^2_2\leq  \norm{f}_1$. Lastly,  $(d)$ holds because of the following chain of inequalities
\begin{align}\label{eq:normdiff_leq_eps}
\Big|\norm{f}_{1, D}-\norm{h}_{1, D}\Big| \leq \norm{f-h}_{1, D}\leq \norm{f-h}_{2, D},
\end{align}
where the first  is due to the Minkowski's inequality for $1$-norm and 
the second  is due to Holder's. 

Next, we show that the quantity $\big\| h-\sign[h_\mathcal{J}]\big\|_{2, D}$ without the square is upper bounded by the same term as in the right-hand side of \eqref{eq:temp1}. That is 
\begin{align}\label{eq:bound on (II)}
\text{(II)} &= \big\| h-\sign[h_\mathcal{J}]\big\|_{2, D}\leq \lambda_1\deq 1+ 2 \norm{f-h}^2_{2, D}+2\norm{f-h}_{2, D}.
\end{align}
The argument is as follows: if $\big\| h-\sign[h_\mathcal{J}]\big\|_{2, D}$  is 
less than one, then the upper bound holds trivially as $\lambda_1\geq 1$; 
otherwise, this quantity is less than its squared and, hence, the upper-bound holds.

Now combining \eqref{eq:bound on (II)}, \eqref{eq:fbar-sign} and \eqref{eq:mismatch_2norm} gives
 \begin{align}\label{eq:bound f sign h}
 \norm{f-\sign[h]}_{2, D}^2  &\leq \norm{f-h}^2_{2, D}+1+\norm{h}_{2, D}^2-2\norm{h}_1 +2\lambda_1\norm{f-h}_{2, D}.
\end{align}
From this bound and \eqref{eq:pe 2 norm stoch}, the error probability satisfies:
\begin{align}\label{eq:up1}
4\PP\Big\{Y&\neq \sign[
h(Z)]\Big\} \leq 2-2\norm{h}_{1, D}+ \underbrace{ \norm{h}_{2, D}^2- \norm{f}_{2, D}^2}_{\text{(III)}}+\norm{f-h}^2_{2, D} +2\lambda_1\norm{f-h}_{2, D}. 
\end{align}
In what follows, we bound the term denoted by (III).\\
\noindent\textbf{Bounding (III):} From the Minkowski's inequality for $2$-norm, we have
\begin{align*}
\norm{h}_{2, D}^2 &\leq \Big(\norm{f}_{2, D}+\norm{h-f}_{2, D}  \Big)^2\\
&=\norm{f}^2_{2, D}+\norm{h-f}^2_{2, D}+2\norm{f}_{2, D} \norm{h-f}_{2, D}\\
&\leq \norm{f}^2_{2, D}+{\norm{h-f}^2_{2, D}+2 \norm{h-f}_{2, D}}
\end{align*}
where the second inequality is due Bessel's inequality implying that $\norm{f}_{2, D}\leq 1$. Hence, the term (III) in \eqref{eq:up1} is upper bounded as
\begin{align}\label{eq:bound on (I)}
\text{(III)}\leq \lambda_2\deq \norm{h-f}^2_{2, D}+2 \norm{h-f}_{2, D}.
\end{align}

As a result of the bounds in \eqref{eq:up1}, \eqref{eq:bound on (I)}, we obtain that 
\begin{align*}
4\PP\Big\{Y &\neq \sign[h(Z)]\Big\}  \leq 2-2\norm{h}_{1, D}+ \lambda_2+\norm{f-h}^2_{2, D}+  2 \lambda_1 \norm{f-h}_{2, D}\\
&= 2-2\norm{f}_{1, D}+2\Big(\norm{f}_{1, D}-\norm{h}_{1, D}\Big)  +\lambda_2+\norm{f-h}^2_{2, D}+ 2 \lambda_1 \norm{f-h}_{2, D}\\
&\leq 2-2\norm{f}_{1, D}+2\norm{f-h}_{2, D}+ \lambda_2+\norm{f-h}^2_{2, D}+  2 \lambda_1 \norm{f-h}_{2, D},
\end{align*}
where the last inequality is due to \eqref{eq:normdiff_leq_eps}. Therefore, from the definition of $\lambda_1$ and $\lambda_2$, and the function $U$ in the statement of the lemma, we obtain 
\begin{align*}
4\PP\Big\{Y\neq \sign[
h(Z)]\Big\}  \leq 2-2\norm{f}_{1, D}+4U(\norm{f-h}_{2, D}).
\end{align*}
This completes the proof by recalling that $f(z)=\EE[Y|z]$ and that from Lemma \ref{lem:MMSE and 01}, $opt_Z=\frac{1}{2}-\frac{1}{2}\norm{f}_{1, D}$.
\end{proof}

%% file: Proofs/proof_lem_norm2fourier.tex
\section{Proof of Lemma \ref{lem:Fourier framework}}\label{proof:lem:Fourier framework}
\begin{proof}
From Lemma \ref{lem:Fourier and loss} in the main text, the generalization error of $g=\sign[\hJ]$ can be written as 
\begin{align*}
     \Loss(g) &= \frac{1}{2}-2^{d-1}\sum_{\mathcal{S}\subseteq [d]} \aS \gS,
 \end{align*}
 Note that since $g$ depends only on the coordinates $\mathcal{J}$, then $\gS=0$ for any $\mathcal{S}\nsubseteq \mathcal{J}$.
 Hence, the above equation simplifies to 
 \begin{align*}
     \Loss(g) &= \frac{1}{2}-2^{d-1}\sum_{\mathcal{S}\subseteq \mathcal{J}} \aS \gS.
 \end{align*}
 Note that $\pS$'s are orthogonal for different $\mathcal{S}$'s and $\sum_{\bfx} \pS(\bfx)^2=2^d$. Hence,  
  \begin{align*}
     \Loss(g) &= \frac{1}{2}-\frac{1}{2} \sum_{\bfx}  \Big( \sum_{\mathcal{S}\subseteq \mathcal{J}} \aS \pS(\bfx)\Big) \Big( \sum_{\mathcal{S}\subseteq \mathcal{J}} \gS \pS(\bfx)\Big)\\
     & = \frac{1}{2}-\frac{1}{2} \sum_{\bfx} \fJ(\bfx) g(\bfx),
 \end{align*}
 where $\fJ\equiv \sum_{\mathcal{S}\subseteq \mathcal{J}} \aS \pS$. By multiplying and dividing $2^d$, the above summation equals to the inner product on the uniform distribution as 
 \begin{align*}
 \sum_{\bfx} \fJ(\bfx) g(\bfx) = 2^d\<\fJ, g\>_\unif.
 \end{align*}
 Hence, with the definition of $g$, we obtain that
\begin{align}\nonumber
 \Loss(g) = \frac{1}{2}-\frac{2^d}{2}\<\fJ,\sign[\hJ]\>
\end{align}
Using a similar argument in deriving \eqref{eq:pe 2 norm stoch}, we can show that 
\begin{align}\label{eq:pe 2 norm stoch fourier}
\Loss(g) =\frac{1}{2} - \frac{2^d}{4}\big(1+\norm{\fJ}_{2,\unif}^2-\norm{\fJ-\sign[\hJ]}_{2,\unif}^2\big).
\end{align}
Notice that this equation is different from \eqref{eq:pe 2 norm stoch} because of the factor $2^d$ and that the norm quantities are taken with respect to the uniform distribution.  We proceed with bounding the $2$-norm quantities. Note that we can apply exactly the same argument used to derive in \eqref{eq:bound f sign h}, as it holds for any underlying distribution.   The $2$-norm quantity above is upper-bounded as follows
\begin{align*}
 &\norm{\fJ-\sign[\hJ]}_{2, \unif}^2  \leq  \norm{\fJ-\hJ}^2_{2,\unif}+1+\norm{\hJ}_{2,\unif}^2-2\norm{\hJ}_{1,\unif} + 2\lambda_1 \norm{\fJ-\hJ}_{2, \unif},
\end{align*}
where $\lambda_1 = 1+ 2 \norm{\fJ-h}^2_{2, \unif}+2\norm{\fJ-\hJ}_{2, \unif}$.
As a result, the loss of $g$ satisfies 
\begin{align*}
\Loss(g) &=\frac{1}{2} - \frac{2^d}{4}\Big(2\norm{\hJ}_{1,\unif} +\underbrace{\big(\norm{\fJ}_{2,\unif}^2-\norm{\hJ}_{2,\unif}^2\big)}_{\text{(I)}} -\norm{\fJ-\hJ}^2_{2,\unif}-2\lambda_1 \norm{\fJ-\hJ}_{2, \unif} \Big).
\end{align*}
Note that $\text{(I)}\geq -\lambda_2$  with $\lambda_2\deq \norm{\hJ-\fJ}^2_{2, \unif}+2 \norm{\hJ-\fJ}_{2, \unif}$ as in \eqref{eq:bound on (I)}.
Next, by adding and subtracting $2\norm{\fJ}_{1, \unif}$, we have that
\begin{align*}
\Loss(g) &\leq \frac{1}{2} - \frac{2^d}{4}\Big(2\norm{\fJ}_{1, \unif}+ 2\underbrace{\big(\norm{\hJ}_{1,\unif}-\norm{\fJ}_{1, \unif}\big)} -\lambda_2 -\norm{\fJ-\hJ}^2_{2,\unif}-2\lambda_1 \norm{f-h}_{2, \unif} \Big)\\
&\leq \frac{1}{2} - \frac{2^d}{4}\Big(2\norm{\fJ}_{1, \unif}- 2\norm{\fJ-\hJ}_{2,\unif} -\lambda_2 -\norm{\fJ-\hJ}^2_{2,\unif}-2\lambda_1 \norm{\fJ-\hJ}_{2, \unif} \Big),
\end{align*}
where we used \eqref{eq:normdiff_leq_eps} to derive the inequality. Therefore, from the definition of $\lambda_1, \lambda_2$ and $U(x)=x^3+\frac{3}{2}x^2+\frac{3}{2}x$ we have that 
\begin{align*}
\Loss(g) &\leq\frac{1}{2} - \frac{2^d}{2}\norm{\fJ}_{1, \unif} + 2^dU(\norm{\fJ-\hJ}_{2,\unif}).
\end{align*}
Lastly, we further bound this expression. Note that $\norm{\cdot}_1 = 2^d\norm{\cdot }_{1, \unif}$. Then, we have that
\begin{align*}
\Loss(g) & = \frac{1}{2} - \frac{1}{2}\norm{\fJ}_{1} + 2^d U( \norm{\fJ-\hJ}_{2, \unif})\leq \frac{1}{2} - \frac{1}{2}\norm{\fJ}_{1} +  U( 2^d \norm{\fJ-\hJ}_{2, \unif}),
\end{align*}
where the last inequality holds by bringing $2^d$ inside $U(\cdot)$. This completes the proof of the lemma.
\end{proof}

%% file: Proofs/proof_thm_L2_polynomials.tex
\section{Proof of Theorem \ref{thm:L2 polynomials}}\label{proof:thm:L2 polynomials}
To derive an upper bound on the empirical error of $\hat{g}$, we first consider a weaker version of the algorithm. The idea is to select $\theta$ randomly instead of optimizing it as in the algorithm. For that, we use Lemma \ref{lem:Pe of sign p-theta} in Section \ref{sec:other classes}. %
%
Consequently, from the lemma and due the fact that $\theta$ in the algorithm is selected to minimize the empirical error, we obtain that 
\begin{align}\label{eq:pemp norm2 bound}
\Pemp\Big\{Y\neq \hat{g}(\bfX)\Big\} &\leq  \frac{1}{2}\norm{Y-\hat{p}}_{2, \Demp}^2,
\end{align}
where $\hat{p}$ is the output of $\Ltwo$-polynomial regression and $\hat{g}\equiv \sign[\hat{p}-\theta]$, as in Algorithm \ref{alg:L2}.
Let $c^*$ be the predictor with minimum generalization error in the $(\epsilon, k)$-approximated concept class. Let $p$ be a degree $k$ polynomial such that $\norm{c^*-p}_2\leq \epsilon$.  Since $\hat{p}$ minimizes the empirical $2$-norm, then the right-hand side of \eqref{eq:pemp norm2 bound} satisfies
\begin{align}\label{eq:Y phat and pstar 2norm}
\frac{1}{2}\norm{Y-\hat{p}}_{2, \Demp}^2 \leq \frac{1}{2}\norm{Y-p^*}_{2, \Demp}^2.
\end{align}

 We proceed by taking the expected error of the empirical error with respect to the random training samples. From \eqref{eq:pemp norm2 bound} and \eqref{eq:Y phat and pstar 2norm} we obtain the following inequalities 
\begin{align*}
\EE\Big[ \Pemp\Big\{Y\neq \hat{g}(\bfX)\Big\}\Big]  &\leq \frac{1}{2}\EE\Big[ \norm{Y-p^*}_{2, \Demp}^2\Big] = \frac{1}{2}\norm{Y-p^*}_{2, D}^2\\
&\stackrel{(a)}{\leq} \frac{1}{2}\Big( \norm{Y-c^*}_{2, D}+\norm{p^*-c^*}_{2, D}\Big)^2\\
&\leq \frac{1}{2}\Big( \norm{Y-c^*}_{2, D}+\epsilon\Big)^2\\
&\stackrel{(b)}{\leq}  \frac{1}{2}\Big( \norm{Y-c^*}^2_{2, D}+ 4\epsilon + \epsilon^2\Big)\\\numberthis \label{eq:up for E_Pemp}
&\stackrel{(c)}{\leq} 2\Pek + \frac{5}{2}\epsilon,
\end{align*}
where (a) holds from Minkowski's inequality for $2$-norm,  (b) holds as $\norm{Y-c^*}_{2, D}\leq 2$, and (c) holds because of the second equality in \eqref{eq:error and inner prod} and that $\Pek=\PP\{Y\neq c^*(\bfX)\}$.

Next, we connect the empirical error of $\hat{g}$ to its generalization error.  Note that the \ac{VC} dimension of all functions of the form $\sign[p]$ for some polynomial of degree upto $k$ does not exceed $d^{k+1}$. Therefore, from VC theory ( See Corollary 3.19 in \citep{Mohri2018}) for any $\delta$, with probability at least $(1-\delta)$, the following inequality holds 
\begin{align}\label{eq:pe vs pe empirical}
\prob{Y\neq \hat{g}(\bfX)} \leq \Pemp\Big\{Y\neq \hat{g}(\bfX)\Big\}  &+\sqrt{\frac{2~d^{k+1}}{n}\log\frac{en}{d^{k+1}}}
+\sqrt{\frac{\log\frac{1}{\delta}}{2n}}.
\end{align}
 Therefore, the proof is complete by taking the expectation and combining it with the last bound in \eqref{eq:up for E_Pemp}.

%% file: Proofs/proof_lem_Pe_theta.tex
\subsection{Proof of Lemma \ref{lem:Pe of sign p-theta}}\label{proof:lem:Pe of sign p-theta}
Note that $y\neq \sign(p(\bfx)-\theta)$, if $\theta$ is between $y$ and ${p}(\bfx)$. Hence, the expected empirical error of $\sign[p(\bfX)-\theta]$ with respect to the random $\theta$ equals to   
\begin{align}\nonumber
\EE_\theta\Big[\Pemp\Big\{Y&\neq \sign[p(\bfX)-\theta] \Big\}\Big]\\\nonumber
&=\frac{1}{n}\sum_{i}\EE_\theta\Big[\11\big\{y_i\neq \sign(p(\bfx_i)-\theta)\big\}\big]\\\label{eq:y_neq_ghat}
&=  \frac{1}{n}\sum_{i}\underbrace{\prob{  \theta \in [{p}(\bfx_i),  y_i]\medcup [y_i, {p}(\bfx_i)]}}_{\PP_i}.
\end{align}
Next, we show that $\PP_i\leq \frac{1}{2}(y_i-p(\bfx_i))^2$ for all $(\bfx_i, y_i)$'s. Suppose $y_i=1$. If  $p(\bfx_i)>1$, then $\PP_i=0$ as $\theta \leq 1$. If $p(\bfx_i)\in [0,1]$, then 
\begin{align*}
\PP_i &= \prob{\theta \in [p(\bfx_i), 1]} =\int_{p(\bfx_i)}^1 (1-t)dt\\
 & =   \frac{1}{2}\big(1-p(\bfx_i)\big)^2 = \frac{1}{2}\big(y_i-p(\bfx_i)\big)^2 . 
\end{align*}
If $p(\bfx_i)\in [-1,0]$, then
\begin{align*}
\PP_i &=\prob{\theta \in [p(\bfx_i), 1]} = \int_{p(\bfx_i)}^1 1-|t| dt\\
& = \frac{1}{2} +  \int_{p(\bfx_i)}^0 (1+t)dt\\
& = \frac{1}{2} - p(\bfx_i) -  \frac{1}{2}(p(\bfx_i))^2\\
& \leq \frac{1}{2}(1+|p(\bfx_i)|)^2= \frac{1}{2}(y_i-p(\bfx_i))^2. 
\end{align*}
Lastly, if $p(\bfx_i)< -1$, then $\PP_i=1$ because $\theta \geq -1$. In this case also  $\PP_i \leq  \frac{1}{2}(y_i-p(\bfx_i))^2$. The case for $y_i=-1$ follows by symmetricity. Hence, we obtain the following inequality	 
\begin{align*}
\EE_\theta\Big[\Pemp\Big\{Y\neq \hat{g}(\bfX)\Big\}\Big]\leq \frac{1}{n}\sum_{i} \frac{1}{2}\big(y_i-p(\bfx_i)\big)^2.
\end{align*}
The proof is complete by noting that the right-hand side equals to $\frac{1}{2}\norm{Y-p}_{2, \Demp}^2$.